\documentclass[11pt,table,reqno]{amsart}

\usepackage[ruled, lined, longend, linesnumbered]{algorithm2e}

\usepackage{etoolbox}
\patchcmd{\section}{\scshape}{\bfseries}{}{}
\makeatletter
\renewcommand{\@secnumfont}{\bfseries}
\makeatother

\usepackage{booktabs}       % professional-quality tables
\usepackage{amsfonts,amsmath,amssymb,amsthm}
\usepackage{mathtools}
\usepackage{comment}
\usepackage{parskip}
\usepackage{siunitx}
\usepackage{etoolbox}
\usepackage{enumitem}
\usepackage{graphicx,subfigure}
\usepackage{tabularx}
\usepackage{xcolor} 
\usepackage[numbers,sort]{natbib}
\usepackage{hyperref}
\usepackage{longtable}
\usepackage{afterpage}
\usepackage{textgreek}
\usepackage{mathrsfs}
\usepackage[foot]{amsaddr}
\newtheorem{thm}{Theorem}[section]

\newtheorem{Definition}[thm]{Definition}
\newtheorem{Proposition}[thm]{Proposition}

\newtheorem{remark}[thm]{Remark}

\newcommand{\xhat}{\ensuremath{\widehat{x}}}

\newcommand{\smoothingsigma}{\sigma}

\newcommand{\energy}{\phi}
\newcommand{\fxhat}{\pi}
\newcommand{\cdf}{\Gamma}
\newcommand{\deltax}{\delta}
\newcommand{\xhateps}{XHAT$_\epsilon$}

\newcommand{\sign}{{\rm{sign}}}

\DeclareMathOperator*{\expectation}{\mathbb{E}}

\DeclareMathOperator*{\argmax}{\rm{argmax}}

\def\<{\begin{equation}}
\def\>{\end{equation}}	

\begin{document}

\title[Provable Robust Classification via learned smoothed densities]{Provable Robust Classification via \\learned smoothed densities}

\author{Saeed Saremi$^{1,2}$}
% \address{Redwood Center for Theoretical Neuroscience, University of California, Berkeley}
\address{$^1$NNAISENSE, Inc.}
\address{$^2$Redwood Center for Theoretical Neuroscience, UC Berkeley}
\email{saeed@berkeley.edu}

\author{Rupesh Srivastava$^1$}
\email{rupesh@nnaisense.com}

\date{May 9, 2020}
\begin{abstract}
Smoothing classifiers and probability density functions with Gaussian kernels appear unrelated, but in this work, they are unified for the problem of robust classification. The key building block is approximating the \emph{energy function} of the random variable $Y=X+N(0,\sigma^2 I_d)$ with a neural network which we use to formulate the problem of robust classification in terms of $\widehat{x}(Y)$, the \emph{Bayes estimator} of $X$ given the noisy measurements $Y$. We introduce \emph{empirical Bayes smoothed classifiers} within the framework of \emph{randomized smoothing} and study it theoretically for the two-class linear classifier, where we show  one can improve their robustness above \emph{the margin}. We test the theory on MNIST and we show that with a learned smoothed energy function and a linear classifier we can achieve provable $\ell_2$ robust accuracies that are competitive with empirical defenses. This setup can be significantly improved by \emph{learning} empirical Bayes smoothed classifiers with adversarial training and on MNIST we show that we can achieve provable robust accuracies higher than the state-of-the-art empirical defenses in a range of radii.  We discuss some fundamental challenges of randomized smoothing based on a geometric interpretation due to concentration of Gaussians in high dimensions, and we finish the paper with a proposal for using walk-jump sampling, itself based on learned smoothed densities, for robust classification.
\end{abstract}
\maketitle
\section{Introduction}
\subsection{The art of smoothing.} 	
Core to the problem of \emph{robust classification} is the issue of the  \emph{smoothness} of  classifiers in the ambient space $\mathbb{R}^d$:
 \begin{itemize} 
 \item[ (i.1)]  It is important to note that  \emph{Bayes optimal classifiers}, approximated via \emph{empirical risk minimization} ~\citep{vapnik1992principles}, may not be sufficitently smooth in high dimensions. One can enforce a degree of smoothness in the hypothesis class (e.g. restricting it to \emph{linear} in the extreme case) but this is typically much less in our control when opting for a richer class of functions, say when the classifiers are parameterized by deep neural networks. Adding to the complexity is the random variable $X$ in $\mathbb{R}^d$ that the classification problem\textemdash the mapping from $X$ to discrete labels\textemdash is defined for. \emph{It is clear that the smoothness of a classifier must be viewed in relation to the distribution of $X$ and its concentration in $\mathbb{R}^d$.}
%, in particular its ``geometry'', where its measure is concentrated

%(lack of overfitting)
\item[ (i.2)] The fact that Bayes optimal classifiers may not be ``sufficiently'' smooth goes against our low-dimensional intuitions, where we do associate good generalization to smoothness (a less wiggly decision boundary) but  these low-dimensional intuitions (un)fortunately break down in high dimensions, where in practice \emph{interpolation}, zero or near zero training loss, often leads to good generalization~\citep{belkin2018understand}. However, the common practice of interpolation with heavily \emph{overparametrized} neural networks has turned out to be a recipe for disaster regarding robust classification as exemplified by \emph{adversarial examples}~\citep{biggio2013evasion, szegedy2013intriguing}. See~\citep{belkin2018overfitting} for a rigorous study on this topic.

\item[ (ii.1)] We seem to be mainly left with two choices. The first is to let go of empirical risk minimization as \emph{the} framework for learning. This is advocated strongly in \citep{madry2017towards} (also see~\citep{goodfellow2014explaining}) from the perspective of \emph{robust optimization}~\citep{wald1945statistical} where instead of minimizing the expected risk, one opts for minimizing the expectation over the maximum risk (where the inner maximization is restricted to some bounded set around each data point, the so-called \emph{attack model}). Far from rigorous, but we may view this as \emph{implicitly} regularizing the smoothness of the classifier, where its degree of smoothness is controlled by the strength of the attack model. For $\ell_p$ attacks, this ``strength'' is correlated with $p$ ($\ell_\infty$ being the strongest) and the radius of the $\ell_p$ ball. There is in fact some empirical evidence in support of this implicit regularization viewpoint of adversarial training, see Figure 2 in~\citep{qin2019adversarial} for an example of such studies.

\item[ (ii.2)] The second approach is simpler conceptually and better understood theoretically, where smoothing a classifier is achieved \emph{explicitly} with a Gaussian kernel and more importantly one can prove guarantees for robustness. Given a non-robust ``base classifier'' $h$, a \emph{provably robust} classifier $g_\sigma[h]$ is constructed, where its degree of smoothness is controlled by the \emph{kernel bandwidth} $\sigma$. Although smoothing kernels have a very deep history in machine learning and statistics\footnote{Visit \url{https://francisbach.com/cursed-kernels/} for an introduction.}, e.g. for the problem of \emph{density estimation}~\citep{parzen1962estimation}, the utility of (Gaussian) noise for smoothing classifiers is a recent development under the topic of ``randomized smoothing''~\citep{lecuyer2018certified,li2018second,cohen2019certified}. The strongest result was optained by~\citep{cohen2019certified}, where they proved a tight $\ell_2$ bound for the robustness of the $\sigma$-smoothed classifier $g_\sigma[h]$ that this work builds on. Lastly, the implicit smoothing (via adversarial training) and the explicit one (via Gaussian noise) can be combined ~\citep{salman2019provably} which we build on as well.

\item[ (iii)] As we alluded to earlier, one important aspect of the problem of robust classification is the distribution of $X$ and more importantly its ``geometry'', where in $\mathbb{R}^d$ its measure is concentrated. This is typically put aside since in high dimensions \emph{density estimation} and \emph{generative modeling} are much harder problems than classification. However, there has been recent progress on an \emph{easier} problem of learning the  (unnormalized) \emph{smoothed density} of  $Y=X+N(0,\sigma^2 I_d)$~\citep{saremi2019neural, saremi2018deep} which plays a central role in this work. Next we discuss how to integrate learned smoothed densities together with \emph{empirical Bayes}~\citep{robbins1956empirical} in randomized smoothing.  

%The framework that emerges is conceptually clear, theoretically grounded, and algorithmically modular but it comes at the computational cost of learning smoothed \emph{densities/energy functions} in advance and computing their gradients per data point within the classifier.
% At a high level, this should not come as a surprise: learning the density of $Y$ is an extra information which does appear relevant to 

%(the energy function is defined as the negative log probability density function modulo a constant) 
\end{itemize}

% Empirical Bayes Smoothed Classifiers
\subsection{Empirical Bayes Smoothed Classifiers.} After a  conceptual tour on the problem of robust classification from the lens of smoothing, we outline our main technical contribution on bringing together \emph{randomized smoothing} developed for constructing provable robust classifiers~\citep{cohen2019certified} and~\emph{neural empirical Bayes} developed for approximating unnormalized densities with empirical Bayes~\citep{saremi2019neural}. There is  simplicity and elegance in constructing the robust classifier $g$, and Theorem 1 proved in~\citep{cohen2019certified}, but there are remaining issues, the most important of which is related to the fact that there is a mismatch between the distribution of the random variable $$ Y=X+N(0,\sigma^2 I_d),$$
which the $\sigma$-smoothed classifier $g$ \emph{effectively} sees, and the distribution of the random variable $X$ for which the original classification problem $X\rightarrow \Omega$ was defined. Algebraically, this mismatch is expressed by $ f_Y = f_X * f_N,$
where $f_Y$ denotes the \emph{probability density function} associated with the random variable $Y$, $f_X$ the p.d.f. associated with $X$, etc., but regarding the concentration of $Y$ and $X$, this mismatch is in fact more severe for $d \gg 1$, discussed in~\citep{saremi2019neural} under the subject  ``manifold disintegration-expansion''. Our first quest is to bridge this gap in the framework of randomized smoothing.

% How can we birdge this gap?

%(see Section 2 in~\citep{saremi2019neural} for more discussion on this issue as related to the \emph{concentration of measure} phenomenon)

%, we call \emph{$\sigma$-smoothed energy function} in this paper,

%\footnote{In an abuse of notation, we also use $\phi$ for the energy function of $Y$.}

The theoretical framework to bridge this gap is \emph{empirical Bayes} which was devised for the problem of estimating $X$ from noisy observations $Y$~\citep{robbins1956empirical}; for Gaussian noise $Y=X+N(0,\sigma^2 I_d)$, the \emph{Bayes estimator} of $X$ given the noisy measurement $Y=y$ can be written in closed form~\citep{miyasawa1961empirical}:
\< \nonumber \label{eq:Miy61} \widehat{x}(y) = y + \sigma^2 \nabla \log f_Y(y).\>
In~\citep{saremi2019neural} this empirical Bayes machinery was used to approximate $\nabla \log f_Y$ starting with a neural network parametrization of the energy function of $Y$ ({\it the energy function is defined as the negative log probability density function modulo a constant}). The end result is \<\label{eq:xhat} \widehat{x}(y) = y - \sigma^2 \nabla \phi(y), \>
where $\phi$ is the (learned) energy function of $Y$ (see Remark~\ref{rem:phi}). This (approximation to the Bayes estimator) of $X$ leads to the following definition.

%\begin{remark}
%	
%\end{remark} 

\begin{Definition}[Empirical Bayes Smoothed Classifier] \label{def:xhat} Let $h: X\rightarrow \Omega$ be a classifier defined on $X$, and let $\energy$ be the (learned) energy function of $Y=X+N(0,\sigma^2 I_d)$. 
\begin{itemize}
\item The classifier $\fxhat$ is defined as
	\begin{equation} \label{eq:pi}
 		 \fxhat(\cdot) = h(\xhat(\cdot)),
	\end{equation}
	where $\widehat{x}(\cdot)$ is given by Equation~\ref{eq:xhat}.
\item The associated $\sigma$-smoothed classifier $g_\sigma[\fxhat]$, which we refer to as empirical Bayes smoothed classifier is given by
 \begin{equation} \label{eq:gpi}
g_\sigma[\pi](x) = \argmax_k \mathbb{P}(\fxhat(x+\varepsilon)=k),{\ \rm where\ } \varepsilon \sim N(0,\sigma^2 I_d),
\end{equation}
where $k\in \Omega$ are the class indices. 
 \end{itemize}
\end{Definition}
Note that the noise distribution used in smoothing $\fxhat$ to $g_\sigma[\fxhat]$ must be the same as the one that generated $Y$ from $X$. Therefore $\fxhat$ is defined such that in deriving the $\sigma$-smoothed classifier $g_\sigma[\fxhat]$, the ``base classifier'' $h$ is evaluated at samples from $\widehat{x}(Y),$ in contrast to the vanilla randomized smoothing where the base classifier is evaluated at samples from $Y$. \emph{In essence, $\widehat{x}(Y)$ forms a bridge between $X$ and $Y$, which is captured in the definition of $\pi$ and its associated $\sigma$-smoothed $g_\sigma[\pi]$.} %To unpack this definition it is useful to focus on Equation~\ref{eq:gxhat}:
%$$ \pi(x+\varepsilon) = h(\widehat{x}(x+\varepsilon)) $$

%\begin{remark}
%In this paper, we drop the parameters of $\phi$ for simplicity. The assumption is we have already learned $\phi$ for any $\sigma$ of interest. Also, Strictly speaking, in the expression above the r.h.s is an approximation to l.h.s. but the error is small assuming the neural network $\phi$ is \emph{expressive}.
%\end{remark}
 
 \begin{remark} \label{rem:denoise}
 	The Bayes estimator $\widehat{x}(y)$ can indeed be viewed as a denoising expression to infer the ``clean'' $X$ from the noisy measurement $Y=y$, but note that the Bayes estimator is the least-squares estimator~\citep{robbins1956empirical}. Therefore, the best one can do\textemdash in the least-squares sense\textemdash is to first learn the energy function of $Y$ and then use Eq.~\ref{eq:xhat} to estimate $X$; see ~\citep{saremi2019approximating} for technical details regarding the energy function vs score function paramterization.
 \end{remark}
 
 \begin{remark} \label{rem:phi}
 	In this paper, we drop the learned parameters of $\phi$ for a clean notation. The assumption is we have already learned $\phi$ for any $\sigma$ of interest with DEEN (see Section 4.1 in~\citep{saremi2019neural}). Also in the expression~(\ref{eq:xhat}), the r.h.s is an approximation to the l.h.s., but the error is small assuming we have access to large amounts of unlabeled data and assuming the neural network $\phi$ itself is expressive~\citep{lu2017expressive}.
 \end{remark}

 \subsection{Summary of contributions.} Next we summarize our contributions that build on \emph{empirical Bayes smoothed classifier}, Definition~\ref{def:xhat}. In the list below, the headers indicate the respective sections in the paper.
 \begin{itemize} 
 \item[\bf{\ref{sec:proof}}.] We revisit the $\ell_2$ robustness of two-class linear classifiers $$ h(x)=\sign(\langle w,x\rangle +b)$$ and we prove that $\pi[h]$ can indeed improve their robustness beyond the \emph{margin}. This result is encapsulated in Proposition~\ref{prop:linear}. The analysis is done when $X$ is distributed as a Gaussian, but we also discuss \emph{mixture of Gaussians}. 
 \item[\bf{\ref{sec:mnist1}}.] We put our analysis to test, beyond the  Gaussian-distributed data, by studying the linear classifier on MNIST and we demonstrate that we can improve their \emph{certified $\ell_2$ robust accuracy} significantly, around \emph{50\%} in a range of radii.
 \item[\bf{\ref{sec:XHAT}}.] The Definition~\ref{def:xhat} requires an already trained classifier, which is suboptimal for achieving the highest provable robust accuracy. To remedy that, we outline an algorithmic framework to \emph{learn} the ``base classifier'' $\pi$ by integrating \emph{randomized smoothing}~\citep{cohen2019certified}, \emph{neural empirical Bayes}~\citep{saremi2019neural}, and \emph{smoothed adversarial training}~\citep{salman2019provably}.  The algorithm is named \xhateps~ for the roles played by the Bayes estimator $\widehat{x}(y)$ and the adversarial training defined by $\epsilon$.
   \item[\bf{\ref{sec:mnist2}}.] We test \xhateps~ on MNIST and show that we can improve the $\ell_2$ robust accuracy of the best \emph{empirical defenses} at the time of writing this paper~\citep{schott2018towards}, in particular we achieve a \emph{provable robust accuracy} of (at least) \emph{95\%} at the radius \emph{1.0} and \emph{81\%} at \emph{1.5}. 
%   \item[\emph{\ref{sec:mnist2}}\it)] 
 \end{itemize}
 
% , called XHAT for the central role played by the Bayes estimator,

 \begin{figure}[b!]
 \begin{center}
% \hspace{1cm}	
 \begin{subfigure}[$h(x+\varepsilon)$]{\includegraphics[width=0.35\textwidth]{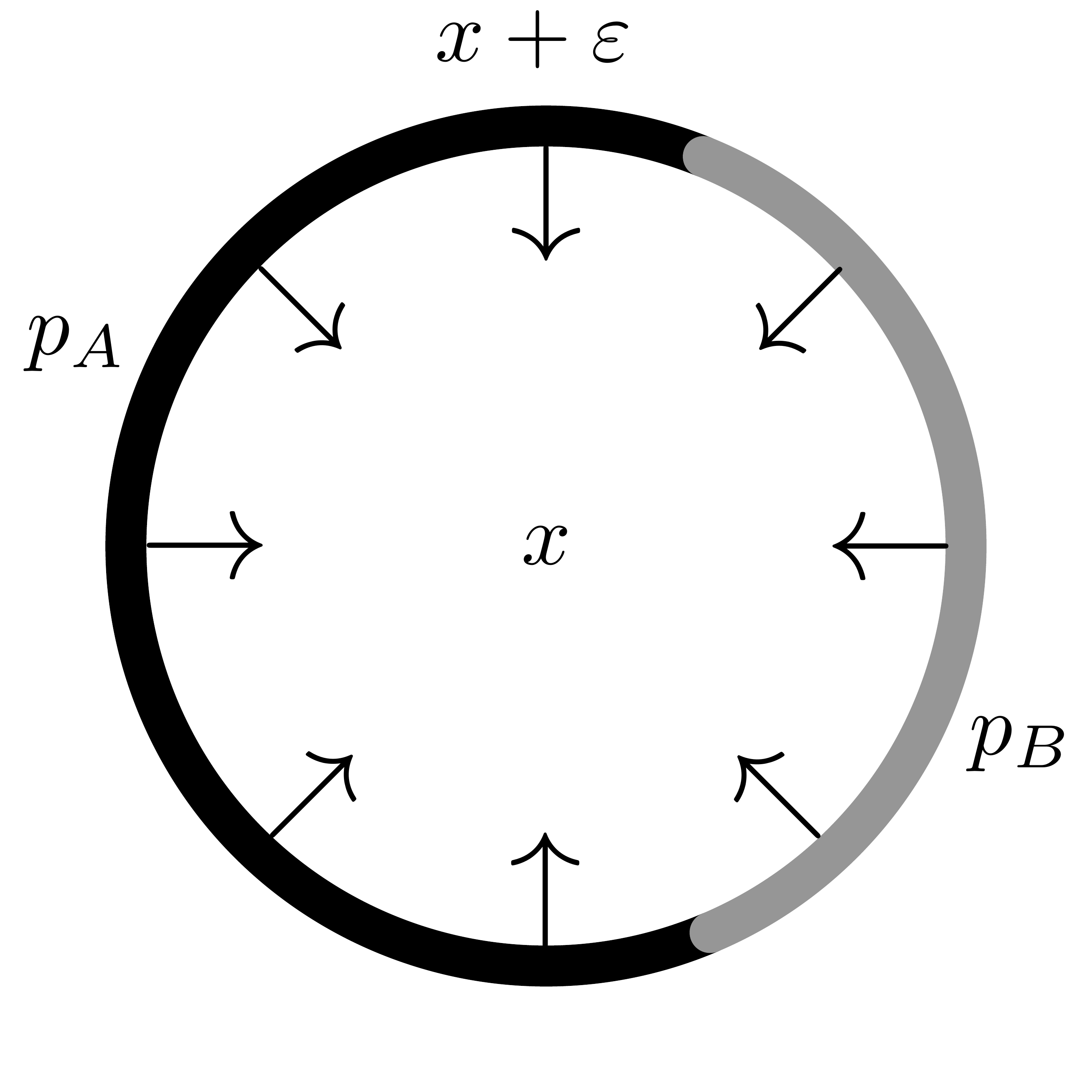}%
 }
 \end{subfigure}
 \begin{subfigure}[$\pi(x+\varepsilon)$]{\includegraphics[width=0.35\textwidth]{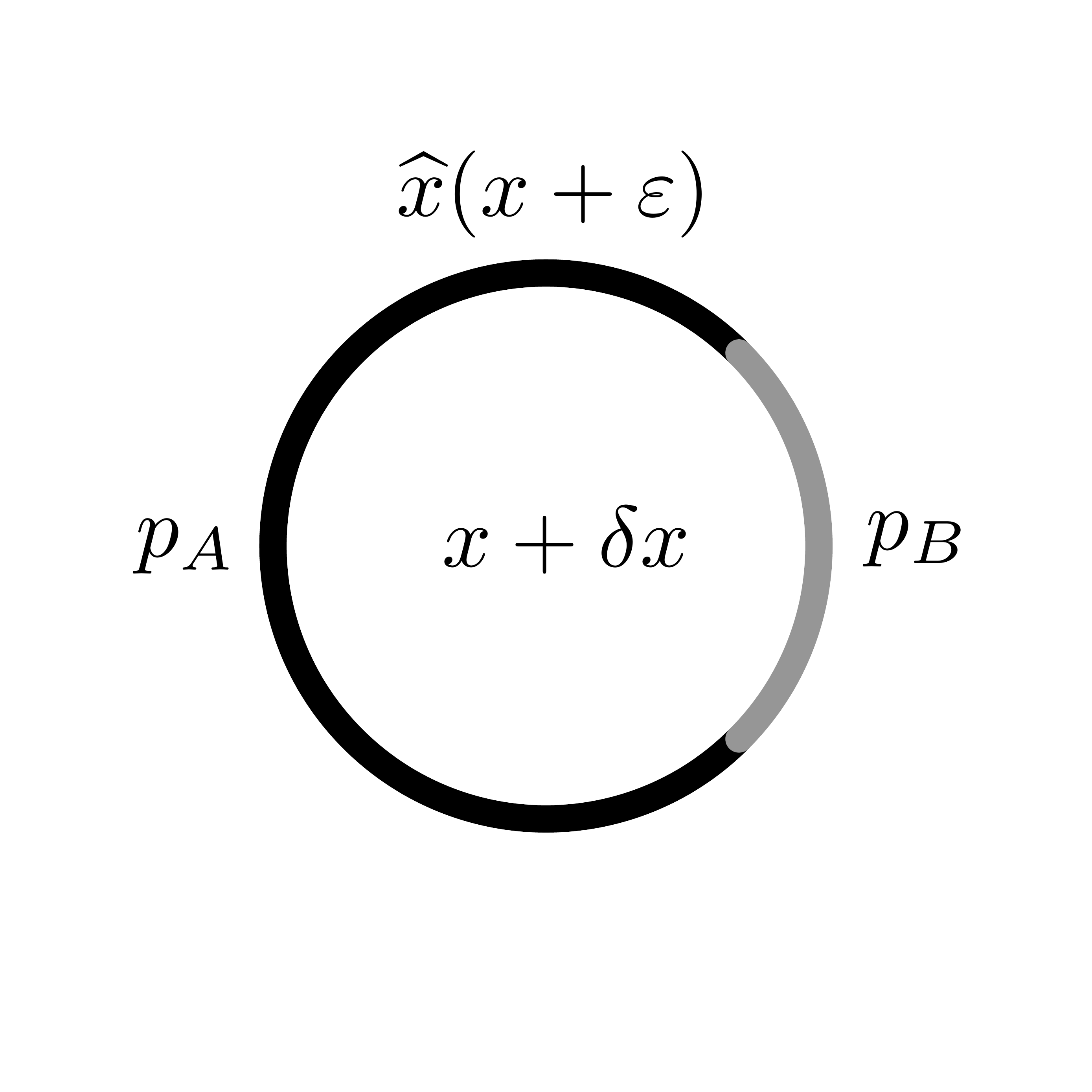}%
 }
 \end{subfigure}
 \caption{  (a)  Schematics of the output of the classifier $h$ at $x+\varepsilon$ where $\varepsilon\sim N(0,\sigma^2 I_d)$, class $A$ (the winner class) is in black and class $B$ in gray, $p_A$ and $p_B$ denote their total probability mass as measured by the Gaussian $N(0,\sigma^2 I_d)$. The arrows represent $-\nabla \phi$ evaluated at $x+\varepsilon$, where $\phi$ is the energy function  of the random variable $Y=X+N(0,\sigma^2 I_d)$. (b) Schematics of Definition~\ref{def:xhat}.  Fixing $x$, the Bayes estimator shrinks the noise, visualized here with a smaller radius for $\widehat{x}(x+\varepsilon)$ in comparison with $x+\varepsilon$ in (a). In addition, there are subtle side effects where $x$ itself ``slides'' to $x+\delta x$ discussed in Section~\ref{sec:noiseshrinkage}. Shrinking the noise is expected to increase $p_A$ in comparison to (a) which results in a more robust classifier. See Remark~\ref{rem:sphere} for more on the schematics.}
 \label{fig:xhat}
 \end{center}
 \end{figure}
 
\clearpage
\section{Empirical Bayes Smoothed Classifiers} \label{sec:MBC}
\subsection{Randomized smoothing.}
To develop some intuitions on the construction of $\pi$ in Definition~\ref{def:xhat}, we consider 
 the two-class linear classifier \begin{equation}  \nonumber \label{eq:linear}
 h(x)=\sign(\langle w,x\rangle +b),	
 \end{equation}
 where $x \sim X$ in $\mathbb{R}^d$, $w \in \mathbb{R}^d$ and $b\in \mathbb{R}$. It is straightforward to see that the linear classifier $h$ is $\ell_2$ robust with the radius given by {\it the margin} at $x$ (the distance to the decision boundary): $$ h(x+\deltax) = h(x)\text{ if } \Vert \deltax \Vert  < r(x),$$
 \< \nonumber \label{eq:margin} r(x) = \vert\langle w,x\rangle +b\vert/\Vert w \Vert, \> 
 where $\Vert \cdot \Vert$ stands for the $\ell_2$ norm. It was shown in~\citep{cohen2019certified} that for the linear classifier $h$, the $\sigma$-smoothed classifier $g_\sigma[h]$ defined by
 \begin{equation} \label{eq:g[h]}
	g_\sigma[h](x)= \argmax_{k} \mathbb{P}(h(x+\varepsilon) = k),\ {\rm where}\ \varepsilon \sim N(0,\smoothingsigma^2 I_d),
\end{equation}
 is identical to $h$ (this is easy to see geometrically by drawing circles for the Gaussian $N(0,\sigma^2 I_d)$ and a line for the decision boundary).

 The construction of $\sigma$-smoothed classifier $g$  (short for $g_\sigma[h]$) in Equation~\ref{eq:g[h]} for any ``base classifier'' $h$ is at the heart of randomized smoothing, where the guaranteed $\ell_2$ robustness:
 $$ g(x+\deltax)=g(x)\text{ if } \Vert \deltax \Vert < r(x) $$
 was proven in consecutive papers~\citep{lecuyer2018certified,li2018second,cohen2019certified} derived from different starting points and with different expressions for $r(x)$. The strongest result was obtained in~\citep{cohen2019certified} with analysis that was based on \emph{statistical hypothesis testing}~\citep{neyman1933ix} and implementations based on~\citep{hung2019rank}. In particular, they derived a tight bound  for $r(x)$ given by the expression
\< \label{eq:r(x)} r(x) = \frac{\sigma}{2} \left(\cdf^{-1}(p_A(x,\sigma))-\cdf^{-1}(p_B(x,\sigma))\right),\>
where $p_A(x,\sigma)$ is the total probability mass of the winner class $k_A$,
$$p_A(x,\sigma)=\mathbb{P}(h(x+\varepsilon)=k_A),$$	
as measured by $\varepsilon \sim N(0,\sigma^2 I_d)$, 
 $p_B(x,\sigma)$ is that of the runner up class, and $\cdf^{-1}$ is the inverse \emph{cumulative density function} of the standard normal distribution $N(0,1)$ (see Figure~\ref{fig:xhat}a for the schematics). For a two-class classifier, $\cdf^{-1}(p_B)=-\cdf^{-1}(p_A)$, Equation~\ref{eq:r(x)} is simplified to \<r(x)= \sigma \cdf^{-1}(p_A(x,\sigma)).\> The expression above for $r(x)$ was computed in~\citep{cohen2019certified} for two-class linear classifiers and it was shown to be identical to the margin. Next, we extend this analysis for $g_\sigma[\pi]$.

\subsection{Two effects of Bayes estimation.} \label{sec:noiseshrinkage} Before analyzing the robustness of $g_\sigma[\pi]$ we first need an expression for the Bayes estimator $\widehat{x}(x+\varepsilon)$.  We start with $$X= N(0,\sigma_0^2 I_d).$$ It follows $$Y =  N(0,(\sigma_0^2 +\sigma^2)I_d),$$ and the Bayes estimator of $X$, $$\xhat(y)=y+ \sigma^2 \nabla \log f_Y(y),$$
simplifies to \< \label{eq:betay} \xhat(y) = \beta y,\>
where the scaling factor $0<\beta<1$ is defined by \< \nonumber \label{eq:beta} \beta^{-1}= 1+(\sigma/\sigma_0)^2.\>

\emph{Fixing} $x$, the noisy samples $$y=x+\varepsilon,~\varepsilon \sim N(0,\sigma^2 I_d)$$ form a Gaussian ball around $x$, and $$\xhat(x+\varepsilon)=\beta x + \beta \varepsilon $$ are visualized by two effects:
\begin{itemize}
	\item[(E1)] (\emph{contraction of the Gaussian ball}) The most prominent effect is the contraction of noise, and it is what empirical Bayes was designed to do in the first place. In this simple setup, it takes the form of scaling the Gaussian ball by the factor $\beta<1$. 
	
	%	In high dimensions $d \gg 1$, the language of~\citep{saremi2019neural}, the radius of the ``$i$-sphere'' (strictly speaking the $x$-sphere here) shrinks from $\sigma\sqrt{d}$ to $\beta\sigma\sqrt{d}$.
	\item[(E2)] (\emph{sliding to low-energy regions})  This ``side effect'' is not simple to analyze in general, where the Gaussian ball $x+\varepsilon$ slides from $x$ (mostly) towards high-density (low-energy) regions in $X$. For $X=N(0,\sigma_0^2 I_d)$, this phenomenon takes the simple form$$ \expectation \xhat(x+\varepsilon) = \beta x.$$
%	Note that $X$ is distributed with center at the origin, and $\beta x$ shifts the Gaussian ball towards the origin.
%	(This is discussed in more details in Section 6 in~\cite{saremi2019neural} from a different angle.) 
\end{itemize}

%In other words,
%	$$ \expectation \xhat(x+\varepsilon) \neq x. $$ 
\begin{remark} \label{rem:sphere}
	It is insightful to consider the geometry of the Gaussian $N(0,\sigma^2 I_d)$ in high dimensions $d\gg1,$ approximated by the uniform distribution on the $(d-1)$ dimensional sphere of radius $\sigma \sqrt{d}$~\citep{vershynin2018high}: \< \nonumber N(0,\sigma^2 I_d) \approx \text{Unif}~(\sigma \sqrt{d} S_{d-1}).\>
	Fixing $x$, $x+\varepsilon$ can be visualized by uniformly distributed samples in the sphere $\sigma \sqrt{d} S_{d-1}$ centered at $x$. In this picture, the estimator $\widehat{x}(x+\varepsilon)$  reduces the radius of the sphere from $\sigma\sqrt{d}$ to $\beta \sigma\sqrt{d}$, and the center of the sphere ``slides'' from $x$ to $\beta x$ closer to the origin. Note that, in general, the Bayes estimation breaks the spherical symmetry of $x+\varepsilon$.  %(Equation~\ref{eq:betay}) 
\end{remark}
%\clearpage
\subsubsection{\bf Mixture of Gaussians.} \label{sec:mixture} The two effects E1 and E2 of  empirical Bayes estimation $\widehat{x}(x+\varepsilon)$ on the Gaussian ball $x+\varepsilon$ centered at $x$ are general phenomena but they become algebraically more complex starting with  a mixture of isotropic Gaussians:
$$ f_X(x) \propto \exp\left(- \frac{\Vert x - \mu \Vert^2}{2\sigma_0^2} \right) + \exp\left(- \frac{\Vert x + \mu \Vert^2}{2\sigma_0^2} \right),$$
where (without a loss of generality) we chose a coordinate system such that $\expectation X = 0$. With this choice, the density of $Y=X+N(0,\sigma^2 I_d)$ is proportional to
$$f_Y(y)  \propto  \exp\left(- \frac{\Vert y - \mu \Vert^2}{2(\sigma^2+\sigma_0^2)} \right) + \exp\left(- \frac{\Vert y + \mu \Vert^2}{2(\sigma^2+\sigma_0^2)} \right).$$
%Following simple algebra, we can derive the expression for 
%$$ \widehat{x}(y) = y + \sigma^2 \nabla \log \mu_Y(y),$$
%given by 
%$$ \widehat{x}(y) = y - \frac{\sigma^2}{\sigma_0^2+\sigma^2} \frac{(y-\mu) \exp\left(- \frac{\Vert y - \mu \Vert^2}{2(\sigma^2+\sigma_0^2)} \right) + (y+\mu) \exp\left(- \frac{\Vert y + \mu \Vert^2}{2(\sigma^2+\sigma_0^2)} \right) }{\exp\left(- \frac{\Vert y - \mu \Vert^2}{2(\sigma^2+\sigma_0^2)} \right)+\exp\left(- \frac{\Vert y + \mu \Vert^2}{2(\sigma^2+\sigma_0^2)} \right)} $$
%\< \widehat{x}(y) = \beta y + \mu \tanh\left(\frac{\langle y,\mu \rangle}{\sigma^2+\sigma_0^2}\right).\>
It follows, %it follows
\begin{eqnarray*}
 \widehat{x}(y) &=& y +\sigma^2 \nabla \log f_Y(y)\\
 &=&  y + \Big(\frac{\sigma^2}{\sigma^2+\sigma_0^2}\Big) \frac{-(y-\mu)M_1-(y+\mu) M_2}{M_1+M_2}\\
 &=& \Big(1-\frac{\sigma^2}{\sigma^2+\sigma_0^2}\Big)y+ \Big(\frac{\sigma^2}{\sigma^2+\sigma_0^2}\Big)\frac{M_1-M_2}{M_1+M_2}\mu
\end{eqnarray*}
where \begin{eqnarray*}
M_1 &=& \exp\left(- \frac{\Vert y - \mu \Vert^2}{2(\sigma^2+\sigma_0^2)} \right), \\
M_2 &=& \exp\left(- \frac{\Vert y + \mu \Vert^2}{2(\sigma^2+\sigma_0^2)} \right),
\end{eqnarray*}
therefore,
$$ \frac{M_1-M_2}{M_1+M_2} = \tanh\left(\sigma_0^{-2}\langle \beta y,\mu \rangle\right). $$

Putting all together, it follows,
\< \label{eq:xhat-mixture} \widehat{x}(y) = \beta y + (1-\beta)\tanh\left(\sigma_0^{-2}\langle \beta y,\mu \rangle\right) \mu ,\>

where $0<\beta<1$ is defined by 
$$ \beta^{-1}= 1+(\sigma/\sigma_0)^2.$$

Equation~\ref{eq:xhat-mixture} for the Bayes estimator is more complex than~(\ref{eq:betay})\textemdash as expected, $\widehat{x}(x+\varepsilon)$ no longer has a spherical symmetry\textemdash but the effects E1 and E2 that were discussed earlier are similar in nature, where the first term, $\beta y$, contracts the noise  and the second term moves the data in the direction of $\mu$ scaled by $(1-\beta)$, to the closest mixture as dictated by the sign of $\langle y,\mu \rangle$. %(It is  intriguing to check the limit $\beta \rightarrow 0.$)

\begin{remark} It is intriguing to consider the limit $\sigma \rightarrow \infty$ ($\beta \rightarrow 0$) where the ``sliding effect'' E2 that we discussed earlier takes an extreme form, where the estimator collapses to the origin, the ``middle ground'' between the mixtures. Fortunately, we are not interested in that regime for robust classification!
\end{remark}

\clearpage
\subsection{Improved robustness of linear classifiers with empirical Bayes.} \label{sec:proof}
Next we prove that the two effects E1 and E2 discussed in the previous section make the analysis of the robustness of the linear classifier nontrivial. %where one can improve the $\ell_2$ robustness beyond the margin (the distance to the decision boundary). %$$  h(x)=\sign(\langle w,x\rangle +b).$$
\begin{Proposition} \label{prop:linear}
	Consider a two-class linear classifier $$ h(x)=\sign(\langle w,x\rangle +b), $$
	and  $X=N(0,\sigma_0^2 I_d)$ (centered at the origin without a loss of generality).
	\begin{itemize} 
		\item[(i)] The empirical Bayes smoothed classifier is given by  \<g_\sigma[\pi](x)=h(\beta x),\>
		where \<\beta^{-1} = 1+(\sigma/\sigma_0)^2.\>
		\item[(ii)] The smoothed classifier $g_\sigma[\pi](x)$ is $\ell_2$ robust with the radius given by the margin evaluated at $\beta x$, multiplied by $1/\beta$:
			$$g_\sigma[\pi](x+\deltax)=g_\sigma[\pi](x)\text{ if } \Vert \deltax \Vert < r(x),$$
			where
			\< \label{eq:prop:r(x)} r(x) = \beta^{-1}  \frac{\vert \langle w, \beta x \rangle + b \vert }{  \Vert w \Vert }  \>
	\end{itemize}
\end{Proposition}
\begin{proof}	
In the proof, $Z$ is the standard Gaussian $N(0,1)$ with values in $\mathbb{R}$, $\mathbb{P}(\cdot)$ are the probabilities calculated under either $\varepsilon \sim N(0,\sigma^2 I_d)$ or $Z$, and $x$ is held fixed.

Start with statement (i) and the case $h(\beta x)=1$:
\begin{eqnarray*}
	h(\beta x)=1 &\equiv& \langle w,\beta x\rangle + b > 0 \\
	&\equiv& \frac{\langle w,\beta x\rangle + b}{\beta \sigma \Vert w\Vert}>0\\
	&\equiv& \mathbb{P}\left(Z< \frac{\langle w,\beta x\rangle + b}{\beta \sigma \Vert w\Vert}\right)>\frac{1}{2} \\
	&\equiv& \mathbb{P}\left( \beta \sigma \Vert w \Vert Z > -\langle w,\beta x\rangle -b \right)  > \frac{1}{2} \\
	&\equiv& \mathbb{P}(\langle w,\beta x + \beta \varepsilon \rangle + b >0)>\frac{1}{2}\\
	&\equiv& \mathbb{P}(\text{sign}\left(\langle w,\beta x + \beta \varepsilon \rangle + b\right)=1)>\frac{1}{2}\\
	&\equiv& \mathbb{P}(h(\beta x +\beta \varepsilon)=1)>\frac{1}{2} \\
	&\equiv&  g_\sigma[\pi](x)=1
%	&\equiv& \mathbb{P}\left(h(\widehat{x}(x+\varepsilon))=1\right)>\frac{1}{2} 
\end{eqnarray*}
The proof follows through starting with $h(\beta x)=-1$ which is equivalent to $g_\sigma[\pi](x)=-1$. This proof is a modification of the calculation in~\citep{cohen2019certified} with the big difference that for vanilla randomized smoothing $g_\sigma[h](x) = h(x)$ but for empirical Bayes smoothed classifier $g_\sigma[\pi](x) = h(\beta x)$.

A short ``geometric proof'' of (i) is given by observing that $\widehat{x}(x+\varepsilon)=\beta x +\beta \varepsilon$
is an isotropic Gaussian that is centered at $\beta x$. It is clear geometrically that the mean of the shifted Gaussian will determine the output of the smoothed classifier:
$ g_\sigma[\pi](x) = h(\beta x).$

To prove (ii) we need to calculate $r(x)=\sigma \cdf^{-1}(p_A(x,\sigma))$, and there are two cases to consider.  Start with 
$g_\sigma[\pi](x)=h(\beta x)=1$ or equivalently $\langle w, \beta x \rangle + b >0 $:
\begin{eqnarray*}
	p_A(x,\sigma) &=& \mathbb{P}(\pi(x+\varepsilon)>0) \\
	 &=& \mathbb{P}(\langle w, \beta x +\beta \varepsilon \rangle + b>0)\\
	 &=& \mathbb{P}(\langle w, \beta  x \rangle - \beta \sigma \Vert w \Vert Z  + b>0)\\
	 &=& \mathbb{P}\left(Z<\frac{\langle w, \beta x \rangle + b}{\beta \sigma \Vert w \Vert}\right ) \\
	 &=& \cdf\left( \frac{\langle w, \beta x \rangle + b}{\beta \sigma \Vert w \Vert} \right)
\end{eqnarray*}
For the case $g_\sigma[\pi](x)=h(\beta x)=-1$ or equivalently $\langle w, \beta x \rangle + b <0 $, the calculation follows through as above:
$$ p_A(x,\sigma) = \cdf\left( -\frac{\langle w, \beta x \rangle + b}{\beta \sigma \Vert w \Vert} \right). $$
Therefore,
$$  p_A(x,\sigma) = \cdf\left( \frac{\vert \langle w, \beta x \rangle + b\vert}{\beta \sigma \Vert w \Vert} \right), $$
and the expression $$r(x) = \sigma \cdf^{-1}(p_A(x,\sigma))$$ reduces to
$$
r(x) = \frac{\vert \langle w, \beta x \rangle + b \vert }{\beta  \Vert w \Vert }.
$$
Note that $\vert \langle w, \beta x \rangle + b \vert / \Vert w \Vert$ is the distance to the decision boundary at $\beta x$ (\emph{not $x$}, as it is the case for vanilla randomized smoothing), and in addition we gain an extra factor $1/\beta = 1+(\sigma/\sigma_0)^2$ compared to the vanilla randomized smoothing where we are ``blind'' to the \emph{density} of $Y=X+N(0,\sigma^2 I_d).$
\end{proof}

In Section~\ref{sec:mixture} we considered the case where $X$ was distributed as a mixture of Gaussians and we found a closed-form expression for $\widehat{x}(y)$, but it was not feasible to repeat the calculations above due to the presence of the second term in Equation~\ref{eq:xhat-mixture}. That said, the effects E1 and E2 that we discussed in Section~\ref{sec:noiseshrinkage} which played major roles in deriving the expression~(\ref{eq:prop:r(x)}) for $r(x)$ are also at play for mixtures of Gaussians, most importantly the contraction of noise by the factor $1/\beta$, but this contraction is not ``clean'' in the case of mixtures due to the presence of $\tanh\left(\sigma_0^{-2}\langle \beta y,\mu \rangle\right)$.

\subsection{Experiments} \label{sec:mnist1}
Proposition~\ref{prop:linear} is indeed limited in  scope due to the assumption made on distribution of $X$, however going through the proof it is clear that there are two main factors at play that are also present for a general $X$: (E1) the contraction of noise schematized in Figure~\ref{fig:xhat}, (E2) the ``sliding'' of the Gaussian ball $x+\varepsilon$ to ``nearest'' low-energy modes of $X$. The first effect is robust and a consequence of the Bayes estimation which is in fact the engine for learning the parameters of $\phi$~\citep{saremi2019neural}. The second effect is subtle though and could have unexpected consequences! For example, in Proposition~\ref{prop:linear}, in the expression $g[\pi](x) = h(\beta x)$, $\beta x$ and $x$ could be on different sides of the decision boundary.

To put these ideas to test beyond Gaussian-distributed data, we looked at the robustness of the linear classifier on MNSIT~\citep{lecun1998gradient}, first  with vanilla randomized smoothing and then with the empirical Bayes smoothed classifier from Definition~\ref{def:xhat}. The results are presented in Figure~\ref{fig:linear} for the randomized smoothing noise levels $\sigma=0.3$ (Figure~\ref{fig:linear}a) and $\sigma=0.6$ (Figure~\ref{fig:linear}b). The dashed lines denote the certified accuracy of $g[h]$ and the solid lines that of $g[\pi]$. The base classifier $h$ was a 1-layer neural network, trained with cross entropy loss with the test accuracy {\it 0.9218}. DEEN is trained separately to learn $\phi$ which is then used for constructing $\pi$ and $g[\pi]$. 
%(see Definition~\ref{def:xhat}).
%{\tt 5e-4}, the test loss {\tt 0.28}, and

 \begin{figure}[h!]
 \begin{center}
 \begin{subfigure}[$\sigma=0.3$]{\includegraphics[width=0.49\textwidth]{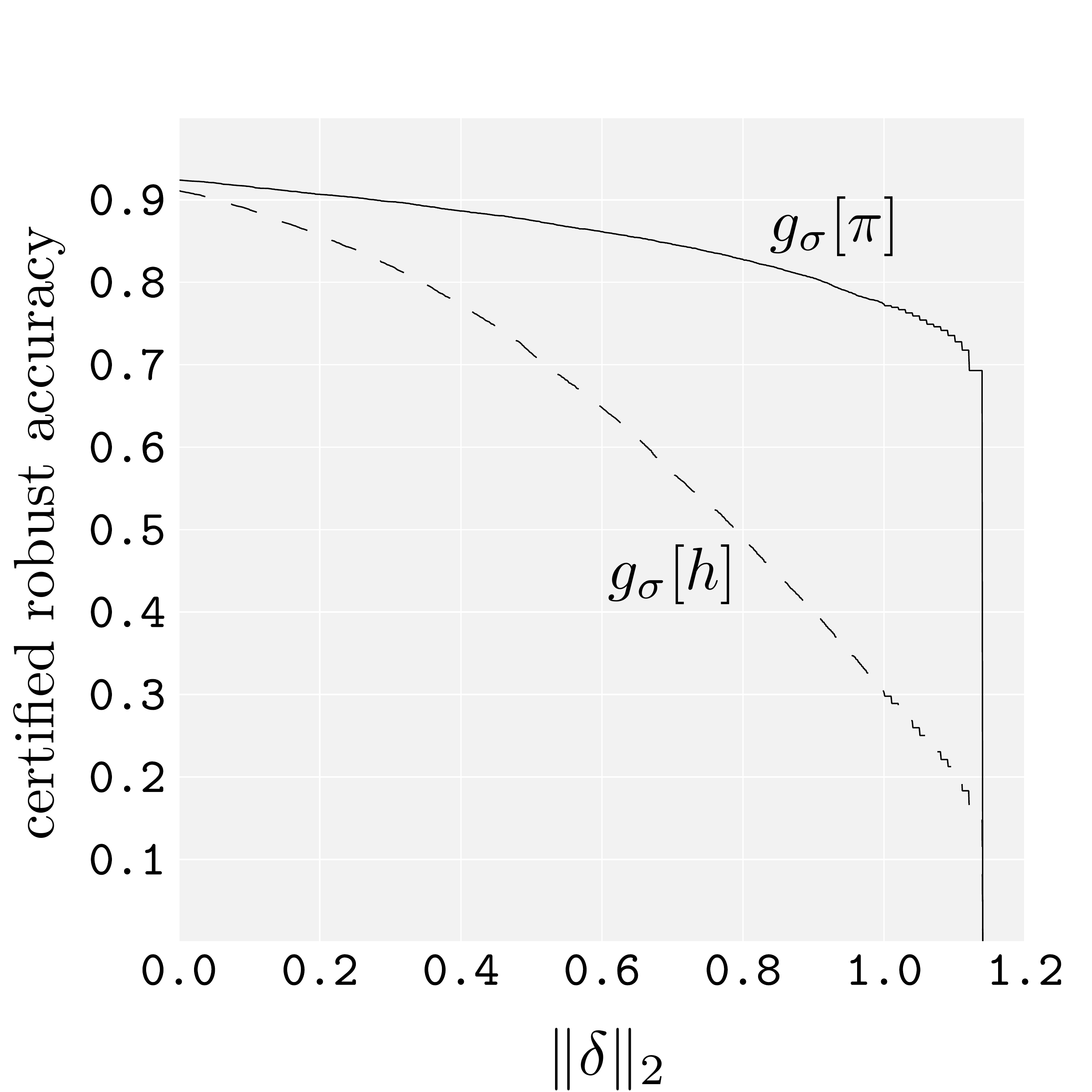}%
 } \end{subfigure} 
  \begin{subfigure}[$\sigma=0.6$]{\includegraphics[width=0.49\textwidth]{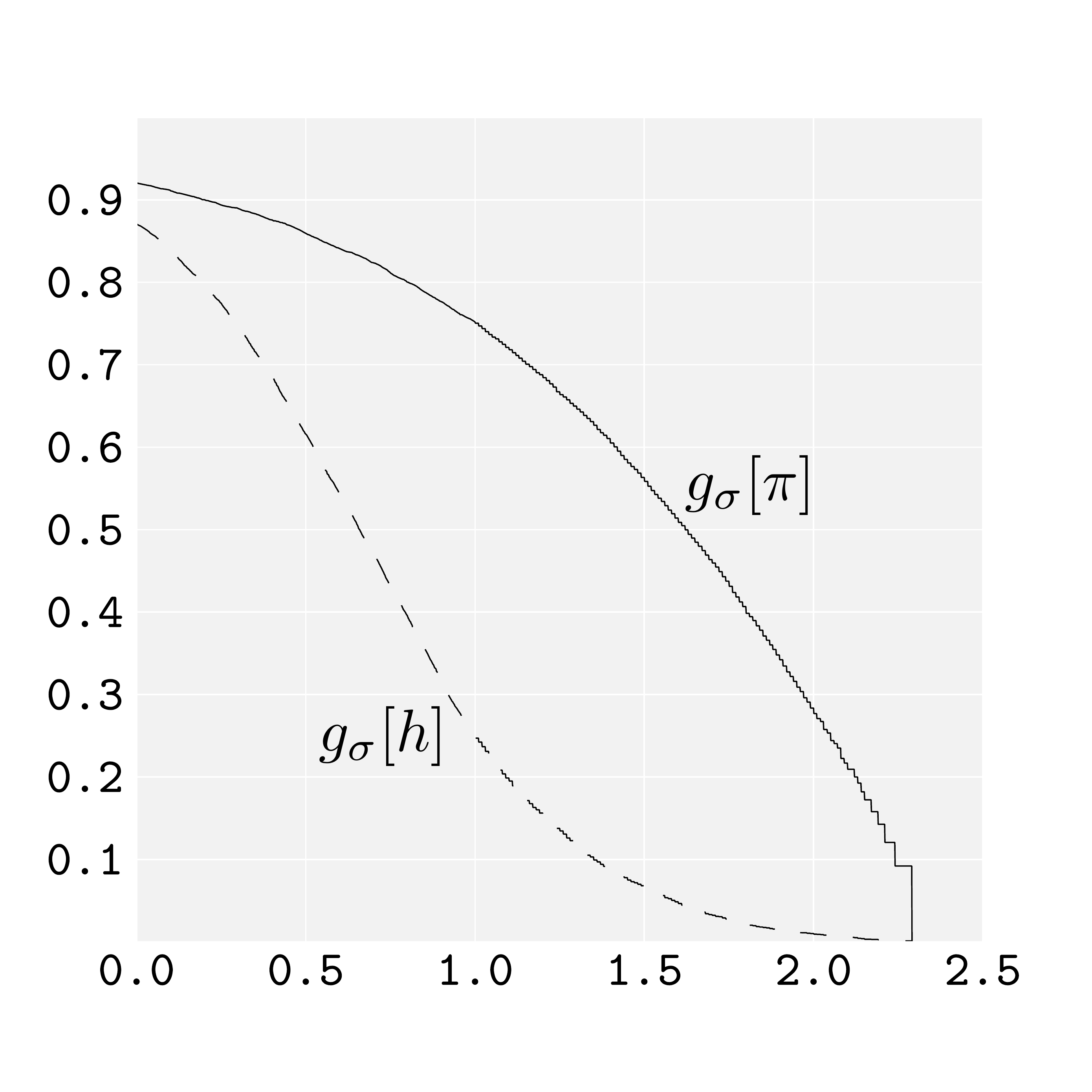}%
 } \end{subfigure}
 \caption{ (\emph{Linear MNIST certified robust accuracy}) The plot of the certified robust accuracy vs. $\Vert \deltax \Vert_2$, the $\ell_2$ radius of perturbations added to inputs, obtained on the test set with randomized smoothing~\citep{cohen2019certified} with \emph{failure probability} $\alpha=0.001$ and $n_c=10^5$ samples used for certification. The base classifier $h$ is the linear classifier, $\pi$ is obtained from Definition~\ref{def:xhat} after first learning $\phi$ with DEEN; $g_\sigma[h]$ and $g_\sigma[\pi]$ are the associated $\sigma$-smoothed classifiers. }\label{fig:linear}
 
% (see Sec.~\ref{sec:deen})
% represented with \emph{dashed} lines and \emph{solid} lines respectively.
 \end{center}
 \end{figure}

\clearpage
 
\begin{table}[h!]
\label{table:mnist-1}
\begin{center}
%\begin{sc}
\begin{tabular}{l| c c c c}
\toprule
$\ell_2$ radius of perturbations& 0.5 & 1.0 & 1.5 & 2.0  \\
\midrule
smoothed linear classifier\  \ $g[h]$  & 71 & 30 & 7 & 1  \\
\emph{empirical Bayes smoothed} $g[\pi]$  & \emph{88}  & \emph{77} & \emph{56} & \emph{28} \\  
\bottomrule
\end{tabular}
%\end{sc}		
\end{center}
\bigskip
\caption{(\emph{MNIST robust accuracy with empirical Bayes smoothed linear classifier}) A compilation of certified robust accuracies from Figure~\ref{fig:linear}, where the highest certified accuracy (in percentage) for $g[h]$ and $g[\pi]$  at several $\ell_2$ radius is reported. Note that the provable accuracies at radii 1.0 and 1.5 are competitive  to accuracies of empirical defences, see Figure 2a in~\citep{schott2018towards} for a collection. The results here are improved significantly by \emph{learning} the empirical Bayes smoothed classifier with adversarial training discussed in Sec.~\ref{sec:main}. For a comparison, see Table~\ref{table:mnist-2}. } 
\end{table}

\bigskip
 \begin{figure}[h!]
 \begin{center}
 \includegraphics[width=\textwidth]{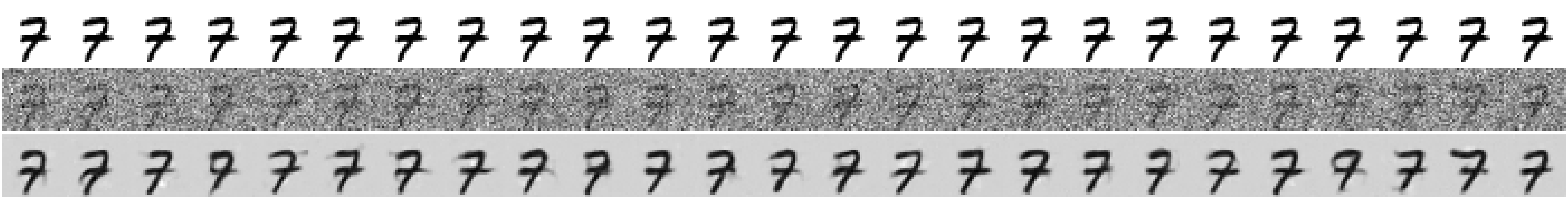}%
 \caption{(\emph{empirical Bayes smoothed classifiers under the hood}) A sample $x$ from MNSIT test set, repeated on the top row. The middle row shows $y_j = x+\varepsilon_j,~\varepsilon_j\sim N(0,\sigma^2 I_d)$, where $\sigma=0.6$; the bottom row shows the corresponding $\widehat{x}(y_j)$ (see Equation~\ref{eq:xhat}); in this example, $g_\sigma[h](x)$ and $g_\sigma[\pi](x)$ are both accurate  with $r(x)=0.0622$ and $r(x)=0.604$, obtained after taking $n_c=10^5$ samples.}\label{fig:vizxhat}
 \end{center}
 \end{figure}
 
 \bigskip
  \begin{figure}[h!]
 \begin{center}
 \includegraphics[width=\textwidth]{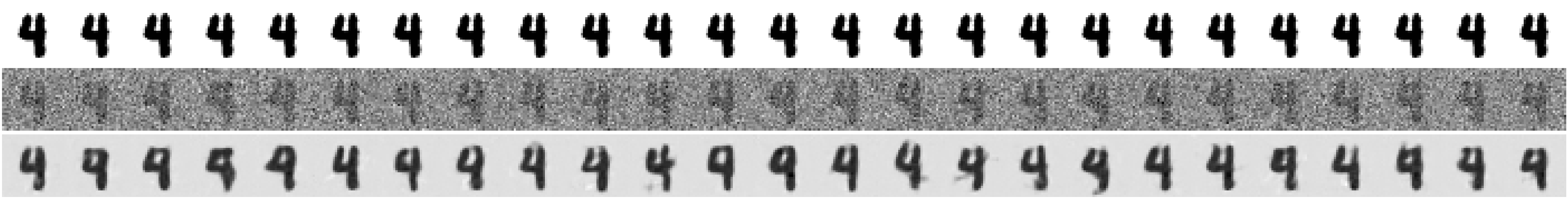}%
 \caption{(\emph{``failure case''}) For $\sigma=0.6$, there are only \emph{26} (out of 10K) cases from the MNIST test set where the certified radius $r(x)$ for smoothed linear classifier ($g_\sigma[h]$) is \emph{higher} than the one for the empirical Bayes smoothed linear classifier ($g_\sigma[\pi]$). Shown here is  one of those 26 cases: $g_\sigma[h](x)$ and $g_\sigma[\pi](x)$ are both accurate with $r(x)=0.55$ and $r(x)=0.40$ respectively. }\label{fig:failure}
 \end{center}
 \end{figure}

%\vspace{0.5cm}

%\subsubsection{\bf DEEN architecture.} \label{sec:deen} Regarding learning $\phi$,  in all the experiments in the paper we used the same architecture described in Section 4.1 in~\citep{saremi2019neural} with the only distinction that the \emph{ConvNet} is significantly smaller, with the expanding convolution channels $(32,64,128)$ instead of $(256,512,1024)$. 

\clearpage

\section{Learning Algorithm}  \label{sec:main}
What we have achieved so far is to develop some intuitions on the effect of the Bayes estimation on robust classification in the framework of randomized smoothing. In particular, we proved in Proposition~\ref{prop:linear} that for \emph{linear classifiers}, we can gain robustness by contracting the noise, a result which is also quite intuitive captured in the schematics of Figure~\ref{fig:xhat}. The Bayes estimation effects that showed up in the proof of Proposition~\ref{prop:linear} are general effects, but assuming a linear classifier is in fact quite limiting. In this sectiom, we provide an algorithm for learning \emph{empirical Bayes smoothed classifiers} by bringing together randomized smoothing~\citep{cohen2019certified}, the neural empirical Bayes~\citep{saremi2019neural},  and smoothed adversarial training~\citep{salman2019provably}. 
%Bringing together randomized smoothing and adversarial training was achieved recently in~\citep{salman2019provably} which this work builds on.

\subsection{Empirical Bayes smoothed adversarial training.} \label{sec:XHAT}
Consider $H$ to be a \emph{soft classifier}, a map from the random variable $X$ in $\mathbb{R}^d$ to probability distributions on the finite set $\Omega$: $$H: X \rightarrow \mathbb{P}(\Omega),$$
which is set up in the context of a \emph{learning problem} for robust classification defined by a loss $\ell_\epsilon(\theta)$ in terms of the parameters $\theta \in \mathbb{R}^p$ (of a neural network) that defines $H$, and the hyperparameter $\epsilon \geq 0$ that controls the robustness-accuracy tradeoff which will become clear shortly. The important step in incorporating empirical Bayes is defining:
	\< \label{eq:Pi} \Pi(x,\theta)= \expectation\left(H\big(\widehat{x}(x+\varepsilon),\theta\big)\right), \>
	where the expectation is over $ \varepsilon \sim N(0,\sigma^2 I_d)$.\footnote{In practice, $\Pi$ is approximated by its Monte Carlo estimate:
	\< \nonumber \Pi(x) \approx \frac{1}{m}\sum_{j=1}^m H(\widehat{x}(x+\varepsilon_j)),\> where $m$ is a hyperparameter.} {\it This definition of the empirical Bayes soft classifier $\Pi$ is a generalization of the definition of empirical Bayes (hard) classifier $\pi$ in Equation~\ref{eq:pi}. There, we imported the parameters of $h$ to set up $\pi$, but here we intend to learn the  parameters of $\Pi$ which is shared with $H$.} 
%	above is inspired by~\citep{salman2019provably}, and it 
\begin{remark}
	It is important to note that, by construction, $\Pi$ defined in~(\ref{eq:Pi}) is indeed a soft classifier, i.e. it is a mapping from $X$ to $\mathbb{P}(\Omega):$
$$\Pi: X \rightarrow \mathbb{P}(\Omega).$$
Also note that $\Pi$ has an explicit dependence on the parameters of the energy function $\phi$ which is not shown here. Throughout, we hid away the parameters of $\phi$. The assumption is that $\phi$ (for any $\sigma$ of interest) is learned in advance.
\end{remark}

In empirical Bayes smoothed adversarial training, we set up a min-max optimization problem to learn the parameters of $\Pi$. The learning problem is defined by the loss
%$$ \theta_\epsilon^* = \argmin \ell_\epsilon(\theta),\text{ where,}$$
	$$ \ell_\epsilon(\theta) = \expectation \max_{\Vert \deltax \Vert \leq \epsilon }  (-\log \Pi_k(x+\delta,\theta)),$$
	which we would like to minimize, where the expectation is over $(x,k)$ pairs, and $\Pi_k$ is the $k$th element of $\Pi$. In practice, the distribution over $(x,k)$ pairs is approximated by the empirical distribution over a training set with $(x_i,k_i)$ elements, where the empirical loss is given by
	\< \label{eq:ell_eps} \ell_\epsilon(\theta) \approx \frac{1}{n}\sum_{i=1}^n \max_{\Vert \deltax \Vert \leq \epsilon }  (-\log \Pi_{k_i}(x_i+\delta,\theta)),\>
	and the loss is optimized with \emph{stochastic gradients descent}:
	$$ \theta \leftarrow \theta - \nabla_\theta \tilde{\ell}_\epsilon(\theta),$$
	where $\tilde{\ell}_\epsilon(\theta)$ is the stochastic loss evaluated on randomly selected mini batches. There are three different types  of gradient evaluations \emph{per parameter update}:
	\begin{itemize}
		\item[(i)] $\nabla_y \phi(y)$: The gradient of the \emph{energy function} $\phi$ in $\mathbb{R}^d$ to evaluate the Bayes estimator \< \nonumber \widehat{x}(y_{ij}) = y_{ij} - \sigma^2 \nabla_y \phi(y_{ij}),\text{ where }y_{ij}=x_i+\varepsilon_j,\>
		which is used to compute 
		\< \label{eq:m} \Pi(x_i,\theta) \approx \frac{1}{m}\sum_{j=1}^m H(\widehat{x}(x_i+\varepsilon_j),\theta).\>		
		\item[(ii)] $\nabla_x \log \Pi_k(x,\theta)$: The gradient of the \emph{empirical Bayes (soft) classifier} $$\nabla_x \log \Pi_{k_i}(x_i,\theta)$$ for the inner maximization problem to evaluate the (stochastic) loss $\tilde{\ell}_\epsilon(\theta)$. This optimization problem  is restricted to the $\ell_2$ ball
		$$ B(x_i,\epsilon) = \{x: \Vert x-x_i\Vert\leq \epsilon\},$$
		and approximated with projected gradient descent (PGD),  where the number of steps taken in PGD is a hyperparameter. \emph{Note that $\nabla \phi$ must be computed at each step of the PGD.	}
		\item[(iii)] $\nabla_\theta \tilde{\ell}_\epsilon(\theta)$: The gradient of \emph{stochastic loss} to update the parameters.
		\end{itemize} 
Following learning, the empirical Bayes (hard) classifier $\pi$ is obtained,
\<\pi(x) = \argmax_k \Pi_k(x,\theta^*),\>
where the implicit dependence of $\theta^*$ (and therefore $\pi$) on $\epsilon$ is understood. The $\ell_2$ (provably) robust classifier $g_\sigma[\pi]$ is constructed as before (see Eq.~\ref{eq:gpi}) with $r(x)$ computed following~\citep{cohen2019certified}. 
%As we discuss at the end of the paper, in the context of provable robust classification, the ``best'' choice of $\epsilon$ also correlates with $\sigma$ used in randomized smoothing.	

\subsection{Experiments.} \label{sec:mnist2}
Here, we revisit the MNIST experiments in Section~\ref{sec:mnist1}. There,  we looked at the provable robust accuracy of the \emph{empirical Bayes smoothed classifier} that was constructed from a \emph{linear classifier}. That ``simple'' construction brought us somewhat close to the state of the art \emph{empirical defenses} (see Figure 2a in~\citep{schott2018towards} for a compilation of several defenses). In this section we present results for \emph{learning} the empirical Bayes smoothed classifier as set up by the loss given in Equation~\ref{eq:ell_eps}. The framework to learn  \emph{empirical Bayes smoothed classifiers} with adversarial training is referred to as \xhateps~ due to the central roles played by the Bayes estimator $\widehat{x}(y)$ and the attack model set by $\epsilon$ in the learning algorithm.

The results are presented in Figure~\ref{fig:mnist-2}, comparing \xhateps~ with XHAT$_0$ ($\epsilon$=0) and also with ~\emph{SmoothAdv}~\citep{salman2019provably}.  As observed in Figure~\ref{fig:mnist-2}, the best certified accuracies are obtained for \xhateps; they are aggregated in Table~\ref{table:mnist-2} and compared with the best empirical robust accuracies reported in~\citep{schott2018towards}; \emph{over three ranges of radii, \xhateps~ provable accuracy improves the reported state-of-the-art empirical accuracies.} The results for \xhateps~  can be improved in a straightforward fashion by taking more noisy samples for certification, but $n_c=10^6$ samples per $x_i$ is already an expensive computation and taking more samples comes with diminishing returns as we explain in the next section. The indirect way to improve the results presented here, especially in the range of radii [1.5,~2.0], is to explore \xhateps~ for larger $\sigma$ as the robust $\ell_2$ radius $r(x)$ scales linearly with $\sigma$. Increasing $r(x)$ in this fashion is subtle though since there will be accuracy tradeoffs. We did not do an exhaustive hyperparameter search for \xhateps~ which could also improve the results presented here. The code will be made public for such explorations.

In the next section, we step back and discuss the fundamental challenges of randomized smoothing in very high dimensions. We also discuss some ideas to go beyond the single-step Bayes estimation of the \emph{empirical Bayes}  with an ensemble of smoothed densities at different $\sigma$.
 \begin{table}[b!]
\caption{(\emph{certified $\ell_2$-robust accuracy of \xhateps~ versus the empirical defense ``analysis by synthesis'' (ABS) ~\citep{schott2018towards} robust accuracy on MNIST test set}). Here, we aggregate the highest certified accuracies of \xhateps~ obtained by taking $n_c=10^6$ noisy samples (see Fig.~\ref{fig:mnist-2}). The results are compared with the best empirical defense results presented in~\citep{schott2018towards}.}
\label{table:mnist-2}
\begin{center}
%\begin{sc}
\begin{tabular}{l| c c c c}
\toprule
$\ell_2$ radius of perturbations& 0.5 & 1.0 & 1.5 & 2.0  \\
\midrule
\emph{certified robust accuracy,}~ \xhateps~  & \emph{98} & \emph{95} & \emph{81} & \emph{57}  \\
empirical robust accuracy, ABS & 92  & 85& 80 & 67 \\  
\bottomrule
\end{tabular}
%\end{sc}
\end{center}
\end{table}

 \begin{figure}[b!]
 \begin{center}
%  \begin{subfigure}[$\sigma=0.6,~n_c=10^5$]{\includegraphics[width=0.45\textwidth]{figs/mnist/xhat_0-xhat_eps-0p6.pdf}%
% }
% \end{subfigure} 
  \begin{subfigure}[$\sigma=0.3,~\epsilon=1.0$]{\includegraphics[width=0.45\textwidth]{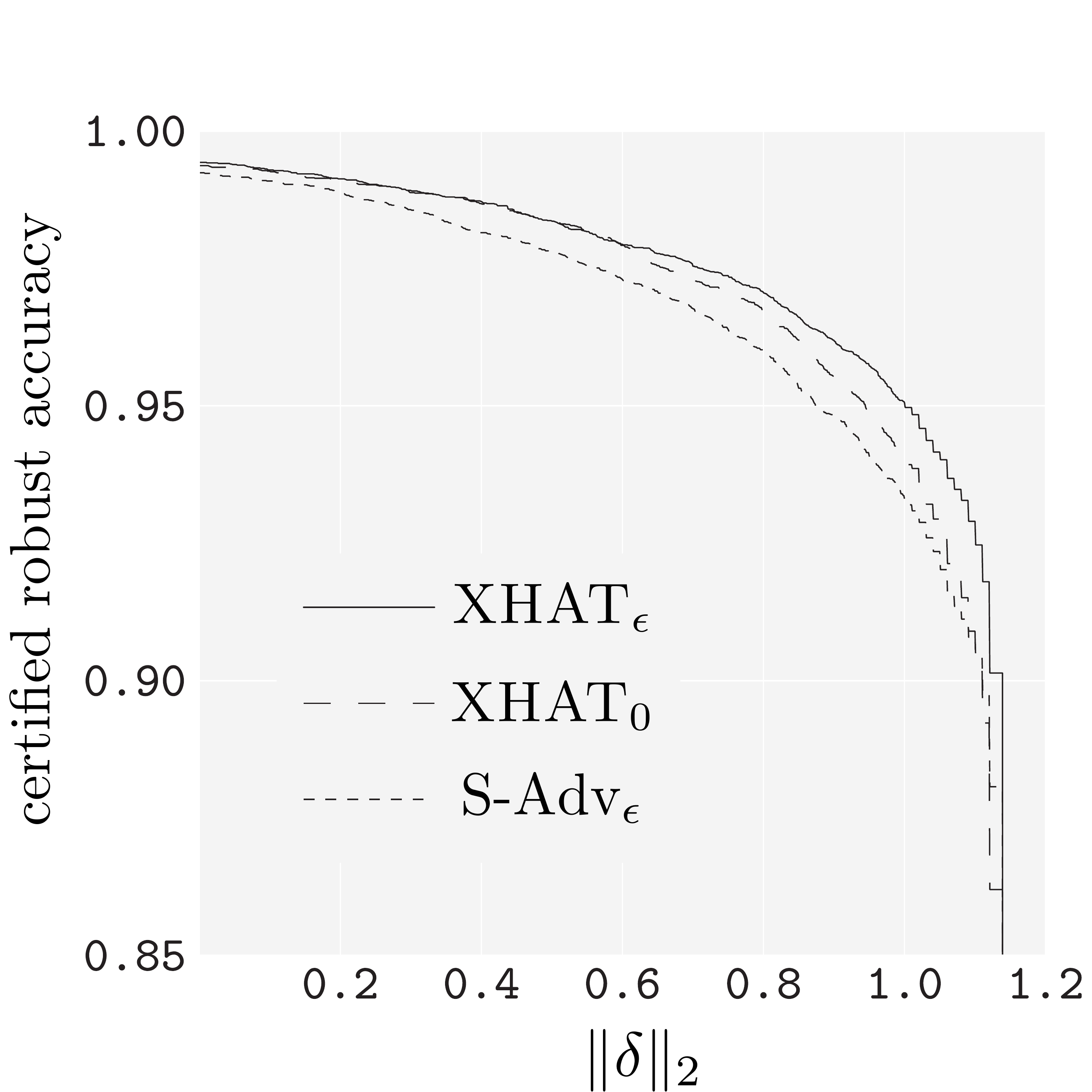}%
 }
 \end{subfigure} 
 \begin{subfigure}[$\sigma=0.6,~\epsilon=1.0$]{\includegraphics[width=0.45\textwidth]{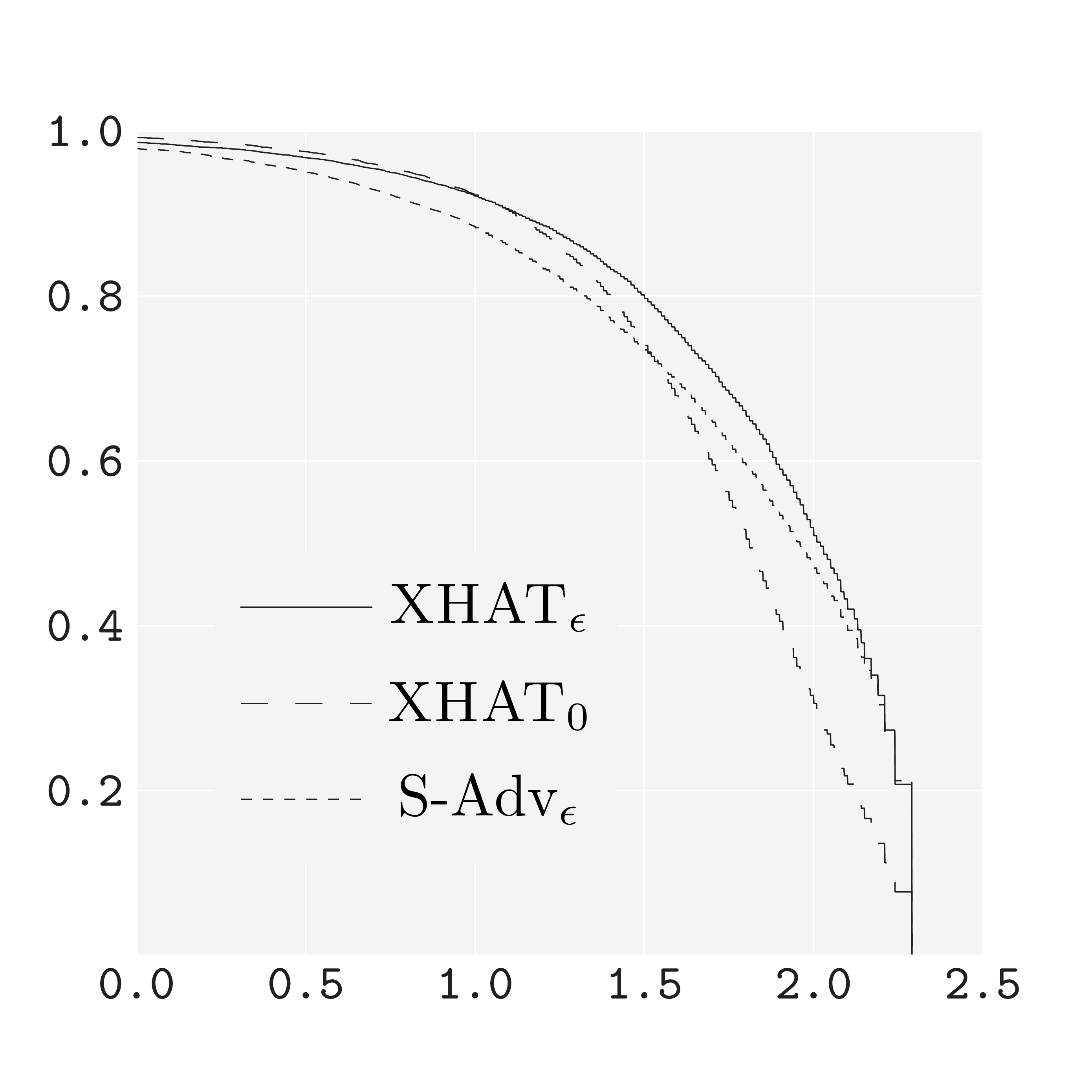}%
 }
 \end{subfigure}	
  \begin{subfigure}[$\sigma=0.3,~\epsilon=1.0$]{\includegraphics[width=0.45\textwidth]{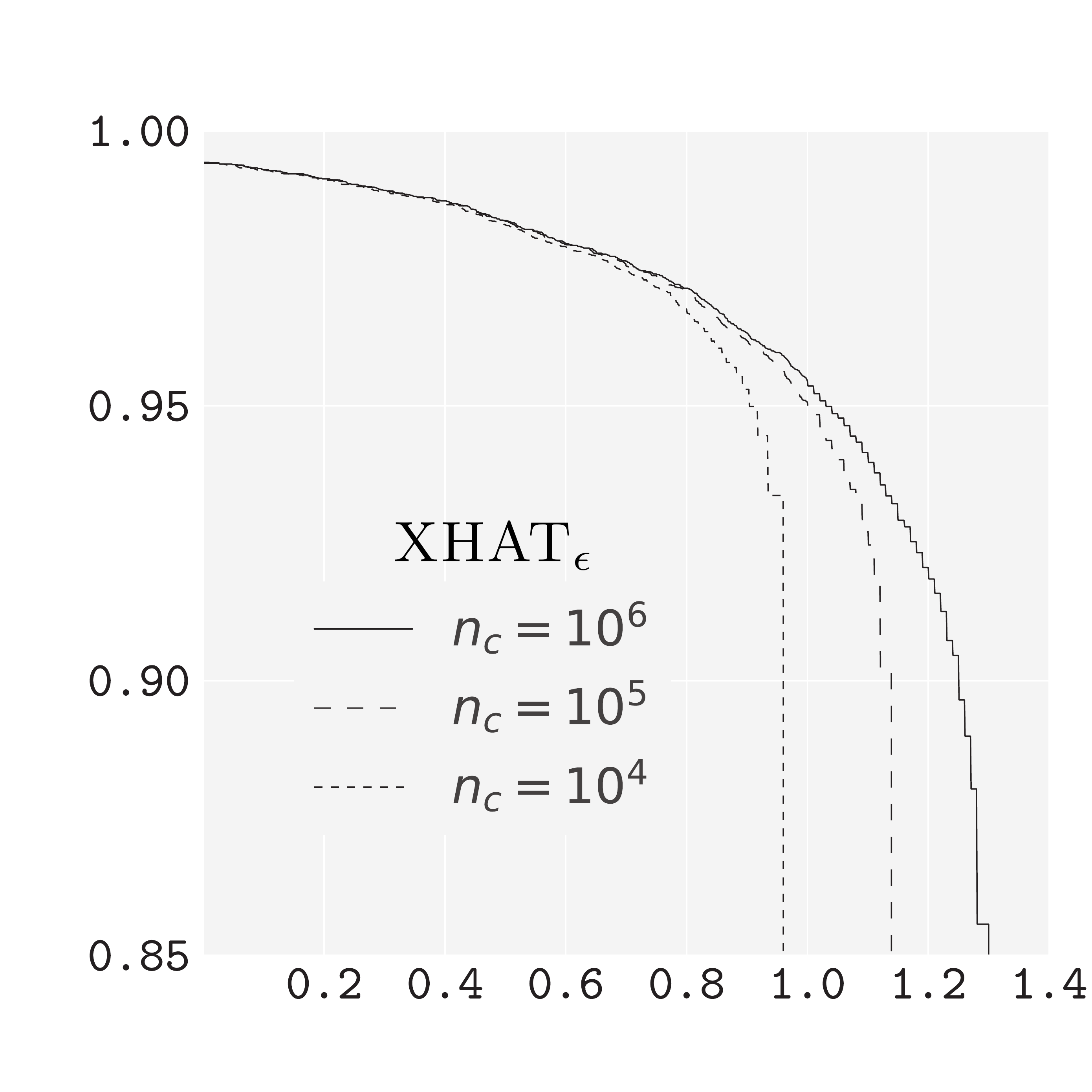}%
 }
 \end{subfigure}
 \begin{subfigure}[$\sigma=0.6,~\epsilon=1.0$]{\includegraphics[width=0.45\textwidth]{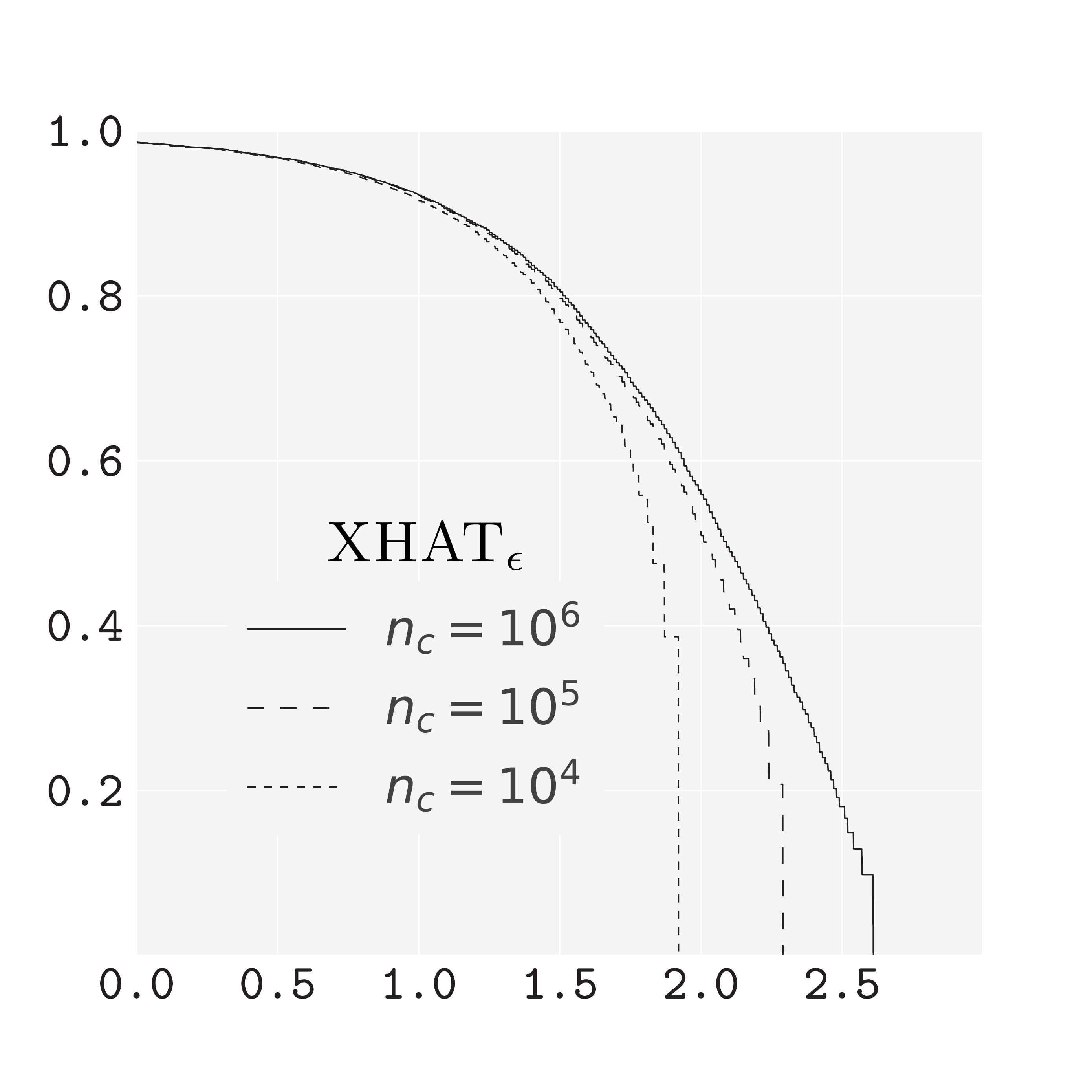}%
 }
 \end{subfigure} 

 \caption{ {\it (empirical Bayes smoothed classifier learned with adversarial training on MNIST)} The certified robust accuracy vs. $\Vert \delta \Vert_2$, the $\ell_2$ radius of the perturbations added to inputs from the test set. \xhateps~ is the algorithm that optimizes the loss $\ell_\epsilon(\theta)$ given in  Equation~\ref{eq:ell_eps}, XHAT$_0$ is the setting $\epsilon=0$, and S-Adv$_\epsilon$ is short for \emph{SmoothAdv}~\citep{salman2019provably} with \texttt{epsilon} set to $\epsilon$. \xhateps~ and \emph{SmoothAdv} are  trained with the same $\epsilon$, hyperparameters, and learning rate schedules: $\it \epsilon=1.0$, $\it m=1$ (see Eq.~\ref{eq:m}), and the number of PGD steps taken is set to \emph{16}. In (\emph{a}) and (\emph{b}), we compare \xhateps~ with XHAT$_0$ and \emph{SmoothAdv} for $\sigma=0.3$ and $\sigma=0.6$. In (c) and (d) we repeat the \xhateps~ experiments by varying $n_c$, the number of noisy samples $\varepsilon_j$ added to each sample $x_i$ for certifying the results (see Fig.~\ref{fig:vizxhat}). The noise values $\sigma=0.3$ and $\sigma=0.6$ was chosen based on their geometric interpretation in terms of the degree of overlap between ``i-spheres'' (see Figure 2 and Table 1 in~\citep{saremi2019neural}). 
 }
 \label{fig:mnist-2}
 \end{center}
 \end{figure}

 \clearpage

\section{Smoothing revisited}
\subsection{Fundamental challenges} \label{sec:missingd} Returning to the conceptual perspective on the ``art of smoothing'' in the introduction, we go over some mostly qualitative arguments regarding randomized smoothing in very high dimensions $d \gg 1$; for the sake of argument assume $d=10^{6}$. This is to highlight both the fact that the expression for $r(x)$ does not scale with $\sqrt{d}$ (the natural $\ell_2$ scale associated with Gaussian noise in high dimensions) and the fact that in high dimensions the \emph{concentration} of $Y=X+N(0,\sigma^2 I_d)$ may be vastly different than $X$. The latter was discussed recently and proven  with some assumptions under the topic ``manifold disintegration-expansion''~\citep{saremi2019neural}.

Consider $d\gg 1$ and assume there are no memory constraints to evaluate $h(x+\varepsilon_j)$ for a single noisy sample $x+\varepsilon_j$, where $\varepsilon_j \sim N(0,\sigma^2 I_d)$. However, we do have \emph{time constraints}  enforced by $n_c$ the number of noisy samples for certification. In randomized smoothing we can certify a radius of up to$$ r_{\rm max}(n_c,\alpha) = \sigma \cdf^{-1}(\alpha^{-n_c}),$$
obtained under the assumption that $h(x+\varepsilon_j)=h(x)$ for every $\varepsilon_j$ sampled in the certification~\citep{cohen2019certified,hung2019rank}, e.g. setting $n_c=10^5,~\alpha=0.001$, we get the maximum radius of $$r_{\rm max}(10^5,10^{-3}) \approx 3.81~\sigma.$$ As discussed in~\citep{cohen2019certified}, relaxing the failure probability $\alpha$ and increasing our budget $n_c$ only comes with diminishing returns, e.g. $$r_{\rm max}(10^{10},10^{-1}) \approx 6.23~\sigma.$$ For simplicity, assume we have a budget s.t. $r_{\rm max}=4\sigma.$ Therefore, to certify a classifier at $\ell_2$ radius $\epsilon$, we should set $\sigma \geq \sigma_{\rm min}$, where $\sigma_{\rm min}= \epsilon/4.$ Note that in high dimensions, the Gaussian is concentrated far away from its mode at the $\ell_2$ norm $\sigma \sqrt{d}$~\citep{tao2012topics}, more precisely:  $$N(0,\sigma^2 I_d) \approx \text{Unif}(\sigma \sqrt{d} S_{d-1}).$$
\emph{Therefore geometrically, in high dimensions, we smooth the classifier at the scale $\sigma \sqrt{d}$ but we can only get certification radius of order $4 \sigma$ (at best) in return; in addition, this mismatch between the scale we smooth the classifier and the radius up to which we can certify the smoothed classifier becomes larger in higher dimensions.} 

This was acknowledged in~\citep{cohen2019certified} but it was put  aside after arguments by visual inspection around Figure 4 in the paper.\footnote{The visual inspection was also augmented by a result on how pooling can in effect gain us a factor of dimensionality that is lost in the vanilla smoothing. This is discussed in the appendix in~\citep{cohen2019certified}. However, it is clear that with pooling \emph{in pixel space} we will also lose in accuracy. In some sense, this discussion is also related to the so-called ``non-robust features''~\citep{ilyas2019adversarial}, which both pooling (in pixel space) and Gaussian noise would (dramatically) affect, especially in higher dimensions.}~A related problem is the concentration of $Y=X+N(0,\sigma^2 I_d)$ compared to $X$. It may appear visually (in \emph{our} visual perception) that for a fixed $\sigma$ the higher resolution images loose less content, but as $d$ increases there are many more directions to ``escape'' the data manifold where $X$ is concentrated.  This has been discussed under the subject  \emph{manifold disintegration-expansion} in~\citep{saremi2019neural} with some analytical results where $\text{dim}(X)=d_{\sharp} \ll d$, and it can be shown that $\text{dim}(Y)= d-1$. From this perspective, ``restoring'' the data manifold with Bayes estimation, one should be able to see more gains  in higher dimensions. However these arguments are far from rigorous and unfortunately very difficult to formalize in high dimensions, e.g. it is not clear how and to what extent the manifold is ``disintegrated'' at \emph{moderately} large dimensions, putting aside the fuzzy notion of ``manifold'' for a data distribution itself.\footnote{Not related to our focus on $\ell_2$ robustness, but there are also fundamental limitations in using randomized smoothing (with Gaussian noise) for certification for $\ell_p$ attacks for $p>2$, where again the dimension $d$ plays a big role~\citep{kumar2020curse}.}
%This is in fact the reasons for the ``cliffs'' seen in the plot of the certified accuracies in Figure~\ref{fig:mnist-2}. 

%\begin{remark}
%
%\end{remark}

%\subsection{Practical challenges.}

%\clearpage

\subsubsection{\bf Other discussions around randomized smoothing}
Despite the conceptual simplicity and the statistical guarantees, practical considerations may dictate that alternatives to the paradigm of randomized smoothing are necessary.
We have seen that constructing the smooth classifier $g$ allows us to obtain statistical robustness guarantees around a data point using a simple sampling procedure.
This Monte Carlo sampling may be computationally expensive (e.g. $10^5$ computations of the base classifier per data point). 
In practice, this may not be a major issue since the \emph{certification} only needs to be performed once before deployment.
But what about computing the \emph{prediction} using a certified $g$ for each new image after deployment?
~\citep{cohen2019certified} suggested that much fewer samples can be used for prediction compared to certification ($10^2$ vs. $10^5$) at the cost of $g$ abstaining from prediction more often.
This is encouraging, but still implies that the cost of deploying the certified classifier is $100\times$ that of using the base classifier per image!
In principle, a potential alternative is to directly train a base classifier whose predictions are provably constant in a well-defined neighbourhood of its inputs, so that smoothing is not required for certification.
This is extremely difficult to accomplish for modern image classifiers based on large and deep neural networks; see \citep{wong2018scaling} and references therein for some initial steps in this direction.
Nevertheless, it is clear that training more robust base classifiers is in general the best recipe for obtaining smoothed classifiers with higher certified accuracy, so there is much to be shared between these two lines of research.

 \subsection{Beyond Bayes estimation.} \label{sec:BBE}
%What we achieved so far is to integrate the Bayes estimator $\widehat{x}(y)$ in randomized randomized smoothing, first on an already trained classifier (in Sec.~\ref{sec:MBC}) followed by an end-to-end adversarial training (in Sec.~\ref{sec:main}).

\emph{Can we do better than the Bayes estimator of $X$, $\widehat{x}(y)=y-\sigma^2 \nabla \phi(y)$, which we have relied on so far?} In this section we put forward some ideas but in pursuing them there are computational challenges for certified robust classification that must be addressed in future research. Indeed the energy function $\phi$ has more utilities than its use for ``denoising''\textemdash the \emph{single-step} empirical Bayes least-squares estimation\textemdash we discuss next. So far, we assumed that the noise level is dictated to us by randomized smoothing and we dropped the implicit dependence of the energy function $\phi$ on $\sigma$, but $\phi_\sigma$ for different regimes of $\sigma$ are qualitatively different as highlighted below for the problem of robust classification. %(In what follows, we assume we have learned an ensemble of energy functions, a set with elements $\phi_\sigma$ that is indexed by $\sigma$.)
\begin{itemize}
	\item[(B1)] One problem with single-step Bayes estimation is the variance of the estimator, which scales as $\sigma^2$. \emph{The first idea is to remove the noise ``as much as possible'' using the attractors of $\phi$}. The energy function $\phi$ has a nice property that its local minima could in principle be used for \emph{memory storage}: this is the notion of associative memory called NEBULA that was introduced in~\citep{saremi2019neural}. NEBULA is governed by the gradient flow that in continuous time takes the form:
	$$ y'(t) = -\nabla_y \phi_\sigma(y(t)),$$
	where the memory is designed to be ``well-behaved'' in some regimes of $\sigma$. For robust classification, one natural idea is to follow up the Bayes estimation with gradient flow to the attractors of $\phi_{\sigma'}$, where $\sigma'$ is typically much smaller than $\sigma$ used in randomized smoothing. \emph{This construction is conceptually intriguing since the Gaussian ball $x+\varepsilon$ will be mapped to finite number of attractors, a measure-zero set by construction.} Instead of $\widehat{x}(y)$ that we relied on in Sections~\ref{sec:MBC} and~\ref{sec:main}, we now have a complex function $A(y,\sigma')$ that takes $y$ sampled from $Y=X+N(0,\sigma^2 I_d)$, and run gradient flow to one of the attractors of $\phi_{\sigma'}$ for $\sigma' \ll \sigma$. As an example, it is straightforward to check that by replacing $\widehat{x}$ with $A$, the analysis of Proposition~\ref{prop:linear} becomes trivial since the attractor $A$ maps $\mathbb{R}^d$ to a single point at the origin. However the attractors (and their dynamics) are in general very complex (see Figures 9 and 10 in~\citep{saremi2019neural}).
		
	\item[(B2)] The second approach is simpler for complex distributions and better understood both theoretically and empirically. It is to follow up the Bayes estimation $$\widehat{x}(y) = y - \sigma^2 \nabla \phi_\sigma(y) $$ with Langevin MCMC:
	\< \label{eq:walk} y_{t+1} = y_t - \delta^2 \nabla \phi_{\sigma'}(y_t)+\sqrt{2} \delta \varepsilon',~\varepsilon'\sim N(0,I_d),\>
	where $\sigma' \ll \sigma$, $\delta$ is the step size and $y_0=\widehat{x}(y)$. After $\tau\gg1 $ steps of Langevin ``walk'', $y_\tau$ is an \emph{exact sample}~\citep{mackay2003information} from $\phi_{\sigma'}$. Therefore we can use the Bayes estimation again\textemdash the ``jump'':
			\< \label{eq:jump} \widehat{x}(y_\tau) = y_\tau-\sigma'^2\nabla \phi_{\sigma'}(y_\tau),\>
			but this time the variance is smaller than before by $(\sigma'/\sigma)^2$.
This is the \emph{walk-jump sampling} introduced in~\citep{saremi2019neural}, but it is now tailored for randomized smoothing. Also, as explained in~\citep{saremi2019neural}, for ``small'' $\sigma$ one does not expect class mixing with a reasonable choice for $\tau$.   Of course, there are many different ways to implement a sampling-based approach to classification, e.g. the empirical defense ``analysis by synthesis''~\citep{schott2018towards} that we earlier compared \xhateps~ against is one example. The formalism there is based on \emph{conditional densities} of \emph{clean samples} but the proposal here is \emph{unconditional} and more importantly it is based on \emph{smoothed densities} (all the way!) which can be readily used in randomized smoothing as visualized in Figure~\ref{fig:walkjump}. 
\end{itemize}

%An example is given in Figure~\ref{fig:walkjump}, where in that example $\sigma=0.6$ and $\sigma'=0.05$. 
\bigskip
 \begin{figure}[h!]
 \begin{center}
 \includegraphics[width=\textwidth]{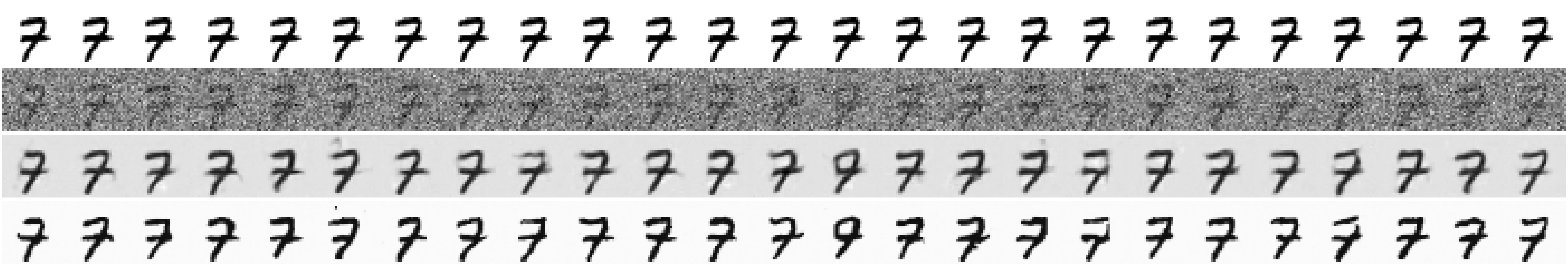}%
 \caption{(\emph{Beyond Bayes estimation}) An example of walk-jump sampling tailored for randomized smoothing. The top row is the ``clean'' sample from Figure~\ref{fig:vizxhat}. The second row shows $y_j=x+\varepsilon_j$, where $\varepsilon_j \sim N(0,\sigma^2 I_d)$, $\sigma=0.6$. The third row shows $\widehat{x}(y_j) = y_j -\sigma^2 \nabla \phi_\sigma(y_j)$ which are the inputs for the walk-jump sampling (Eqs.~\ref{eq:walk} and~\ref{eq:jump}) with $\sigma'=0.05$, $\delta=0.001$, and $\tau=100$. The fourth row shows the final jump samples $\widehat{x}(y_{j,\tau})$. (Note that one of the samples $y_j$ is mistaken for a ``9'', even to a human observer.)}\label{fig:walkjump}
 \end{center}
 \end{figure}

\section{Summary}
We finish with a summary.
\begin{itemize}
	\item We introduced \emph{empirical Bayes smoothed classifiers} and studied it theoretically for linear classifiers. The theoretical results are encapsulated in Proposition~\ref{prop:linear}. The assumptions in the proposition is quite limiting but the proof points to the key factor at play that are general, and on MNIST we showed that an empirical Bayes smoothed \emph{linear} classifier has certified $\ell_2$ robust accuracies in the ballpark of sophisticated empirical defenses.
	\item Motivated by Proposition~\ref{prop:linear}, we introduced the algorithmic framework \xhateps~ to \emph{learn} empirical Bayes smoothed classifiers. We revisited the MNIST results of empirical Bayes smoothed \emph{linear} classifiers and demonstrated that with \xhateps~ on only two values of $\sigma$, we can achieve \emph{provable robust accuracies} higher than the best empirical defenses on a range of radii. Having provable robust accuracies on par with empirical defenses is in general a difficult goal since empirical defenses typically have ``holes''; in the short history of research on  adversarial examples, it has typically been a matter of time (sometimes, few short days) to break a defense~\citep{athalyE2018obfuscated}.\footnote{Visit \url{https://simons.berkeley.edu/talks/tbd-76} for a presentation on this topic.}
	\item  In the closing, we proposed a theoretically plausible framework based on \emph{walk-jump sampling} to go beyond the single-step empirical Bayes estimation, which could potentially improve our results significantly. However, that comes with immense computational challenges due to the nature of statistical guarantees that one needs to meet in  certifying a smoothed classifier. Ultimately, certification (in the framework of randomized smoothing) has deep limitations since the smoothing of classifiers happens at a much larger scale than the radii we can certify them for, and in addition, the machinery is limited due its $\ell_2$ formulation. Robust classification based on \emph{smoothed densities} could also be used in \emph{empirical defenses} but we chose to explore its potentials for provable robust classification. 
	\item At its core, this work was based on~\citep{cohen2019certified} and~\citep{saremi2019neural}, both are very recent developments, and for developing \xhateps~ we greatly benefitted from~\citep{salman2019provably}. Looking at the Bayes estimation from the lens of denoising (see Remark~\ref{rem:denoise}), there are many references that one should consult: adding noise and denoising are indeed the very first ideas that come to mind for defending against adversaries, but if not done right, they are surprisingly brittle themselves~\citep{athalyE2018obfuscated}. See~\citep{xiE2019feature} for a recent successful example. 
	\item This work can also be seen as employing theoretical tools and algorithms in the realm of \emph{unsupervised learning} for the problem of  provable/certified robust classification, which has also been explored recently but from a very different starting point~\citep{carmon2019unlabeled}. This work is the very first attempt to bring in learning \emph{smoothed energy functions} to the problem of (certified) robust classification.      \end{itemize}

\subsection*{Acknowledgments} 
We would like to thank Francis Bach for discussions.

\bibliography{xhat.bib}
\bibliographystyle{alpha}

\end{document}